\documentclass[12pt,twocolumn,notitlepage]{article}

\usepackage[utf8]{inputenc}

\usepackage{subfig}	         
\usepackage{graphicx}	
\usepackage{amsmath, amsthm, amssymb}    	
\usepackage{amsfonts}	

\title{Determination of Digital Straight Segments Using the Slope}
\author{Alejandro Cartas, María Elena Algorri\\
	    \emph{Neuroimaging Laboratory} \\ \emph{Instituto Tecnológico Autónomo de México}}
\date{August 2009}

\newtheorem{theorem}{Theorem}
\newtheorem{lemma}{Lemma}
\newtheorem{corollary}{Corollary}

\begin{document}

\bibliographystyle{plain}

\maketitle

\begin{abstract}
We present a new method for the recognition of digital straight lines based on the slope. This method combines the Freeman's chain coding scheme and new discovered properties of the digital slope introduced in this paper. We also present the efficiency of our method from a testbed.\\

\end{abstract}

\section{Introduction}

The recognition of digital straight segments (DSS for short) is an old and common problem in digital geometry. Consequently, many algorithms have been developed to solve this problem, such as \cite{Buzer2006,Debled-Rennesson1995,Soille1991}. These algorithms are based on particular properties of the digital straight lines (DSL) that have been discovered throw the years, but just a few uses the slope because it's not easy to deal with its discrete nature. For instance, in contrast with a real line, the represented slope $m$ of a DSL could be different from the slope between two distinct points of the same digital line. Using only the first octant, in this paper:

\begin{itemize}
\item We present the theorems that state the boundaries of the slope between any fixed point of the DSL and its successive points (section \ref{sec:slopeBoundaries}).
\item We propose an algorithm for the recognition of DSS that combines the Freeman chain coding scheme and the knowledge about the slope (section \ref{sec:method}).
\item Finally, we show the tests under which the proposed algorithm was submitted, which were intended to quantify their effectiveness and obtain relevant information (section \ref{sec:tests}).
\end{itemize}

\section{Background}
\subsection{Distance $D_8$}

In the digital geometry it is useful to work with other metric spaces other than the Euclidean distance. One of such metric spaces is the distance $D_8$, also known as the chessboard distance or the $L_\infty$ metric. So if $p, q\in \mathbb{R}^2$, $p=(x_0,y_0)$, and $q=(x_1,y_1)$, we can define the distance $D_8$ between $p$ and $q$ as:

\begin{equation}
D_8(p,q)=max\{|x_1-x_0|,|y_1-y_0|\}
\label{eq:distanceD8}
\end{equation}

In the digital space, the first four distances $D_8$ from a fixed grid point are shown in Fig. \ref{fig:firstFourD8Distances}.

\begin{figure}[h]
\centering
\includegraphics[scale=0.15]{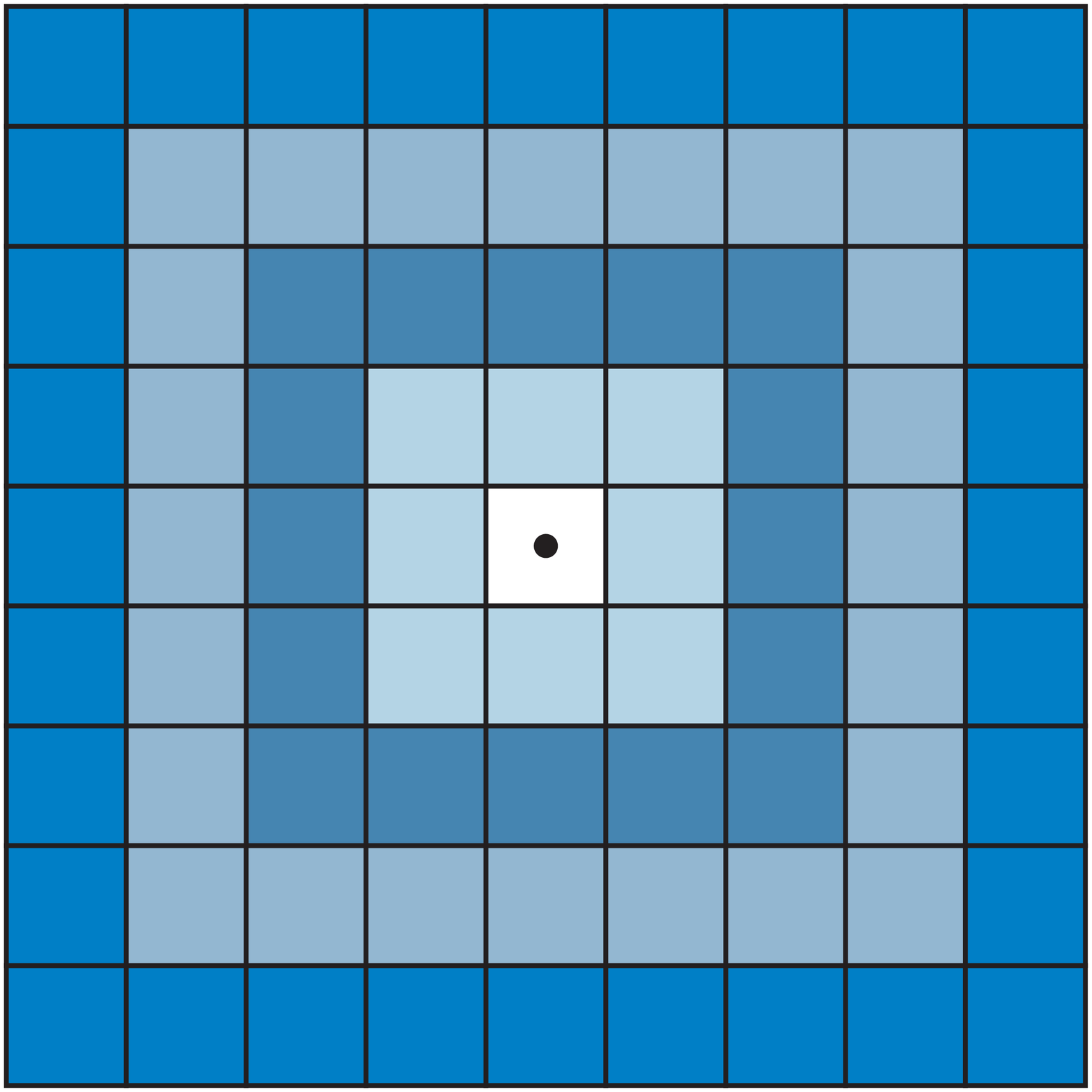}\hspace{5mm}
\includegraphics[scale=0.15]{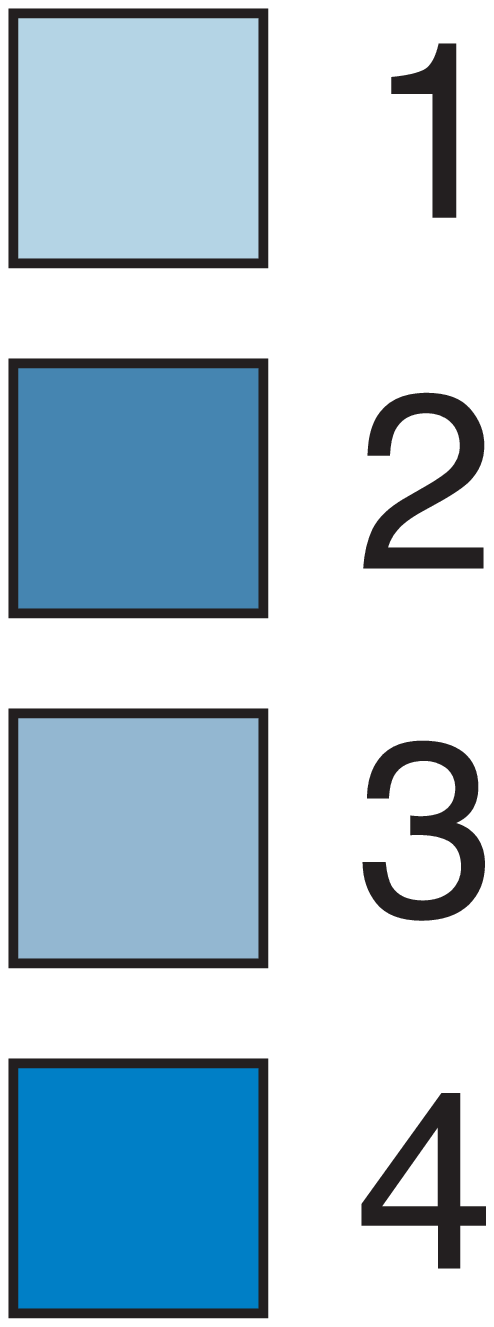}
\caption[Distances $D_8$ = 1, 2, 3, and 4 from the white filled pixel in the middle]{Distances $D_8$ = 1, 2, 3, and 4 from the white filled pixel in the middle.}
\label{fig:firstFourD8Distances}
\end{figure}

\subsection{Problem Statement}

Let $\rho_0, \ldots, \rho_n \in\mathbb{Z}^2$ be the sequence of points that describe the 8-neighborhood boundary of an object, then the problem is to find the set $\mathcal{L}$ of exact DSS that are defined by $\rho_0, \ldots, \rho_n$.

\subsection{Digital Straight Segments}

\begin{figure*}[t]
\centering

\begin{minipage}{\textwidth}
\begin{center}
\subfloat[]{\label{fig:DSS1.1} \hspace{5mm}\includegraphics[scale=0.8]{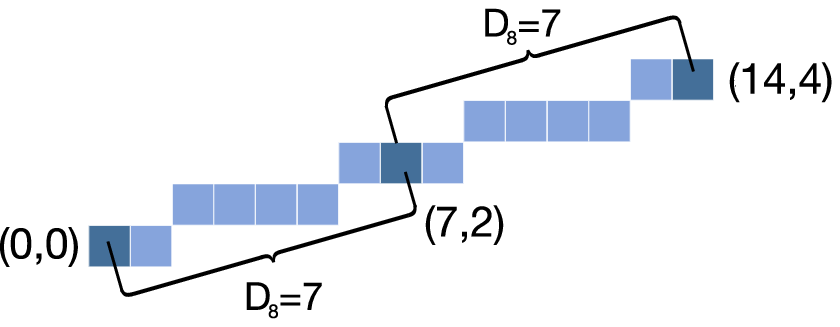}\hspace{5mm}}\hspace{0cm}
\subfloat[]{\label{fig:DSS1.2} \hspace{5mm}\includegraphics[scale=0.8]{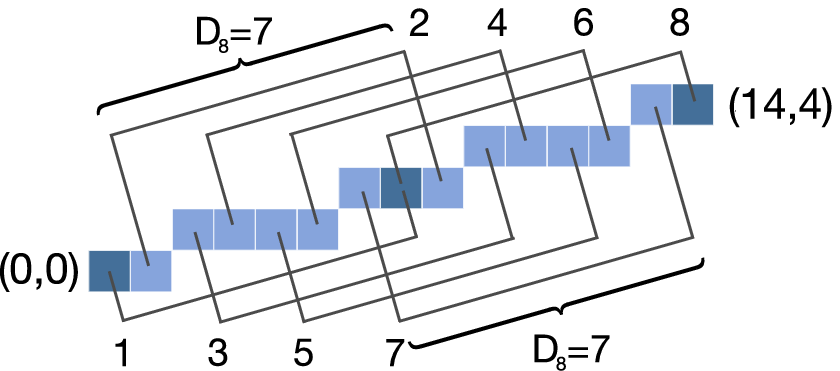}\hspace{5mm}}
\end{center}
\end{minipage}

\subfloat[]{\label{fig:DSS1.3} 
\begin{minipage}{\textwidth}
\begin{center}
\includegraphics[scale=0.8]{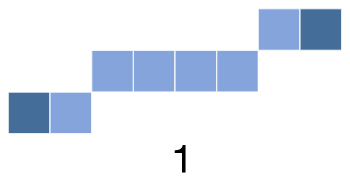}\hspace{6.5mm}
\includegraphics[scale=0.8]{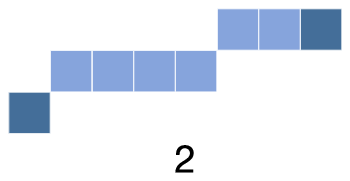}\hspace{6.5mm}
\includegraphics[scale=0.8]{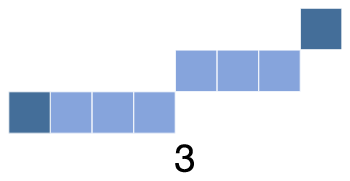}\hspace{6.5mm}
\includegraphics[scale=0.8]{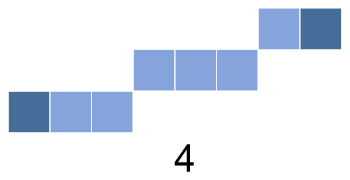}\\
\vspace{5mm}
\includegraphics[scale=0.8]{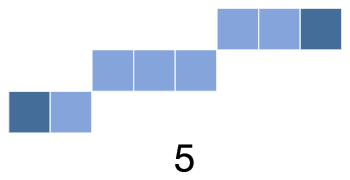}\hspace{7mm}
\includegraphics[scale=0.8]{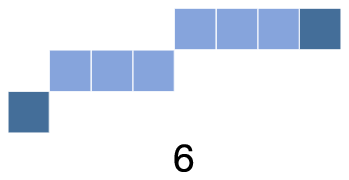}\hspace{7mm}
\includegraphics[scale=0.8]{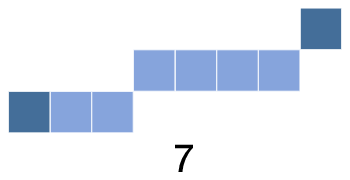}
\end{center}
\end{minipage}
}

\caption[Optimal digital line segment with slope $2/7$ showing two periods]{Optimal digital line segment with slope $2/7$. (a) The DSS showing two periods of length $D_8=7$ between the dark blue filled points. (b) All the paths in the DSS that are defined by two points with distance $D_8=7$ between them and slope equal to $m$. (c) The seven different \textit{basic sets of points} that define the DSS and were extracted from the paths in (b).}
\label{fig:DSS1}
\end{figure*}

Basically, the algorithms for drawing any DSS need a starting point $P(x_0,y_0) \in\mathbb{Z}^2$, a slope $m\in\mathbb{R}$, a direction, and the number of points to draw. In order to achieve a \textit{sense of straightness}, these algorithms must determine the points that are closer to the points of the real line with the same starting point and slope. For DSS with irrational slopes there is only one way to achieve this sense of straightness, in other words, they are described uniquely. In contrast, DSS with rational slopes could have many combinations of points that achieve the sense of straightness. For both kinds of DSS,  the optimal points that give the sense of straightness with respect to $P$ belong to the set
 
\small
\begin{equation}
\{(x,y)\in\mathbb{Z}^2:x>x_0 \wedge y=\lfloor m\cdot(x-x_0)+y_0+1/2 \rfloor\}
\label{eq:optimalCriterion}
\end{equation}
\normalsize

The reason why rational DSS are not uniquely described is that they are periodic. The periodicity means that the points of any rational DSS are defined by a \textit{basic set of points}. Specifically, if the DSS slope $m$ is an irreducible rational fraction equal to $r/s$, then the \textit{basic set of points} is $\{(x_0, y_0), \ldots, (x_s, y_s) \}$, and any point $p$ belonging to the DSS is equal to $(x_i+s\cdot n, y_i+r\cdot n)$, where $i = 0,\dots, s$ and $n \in \mathbb{Z}$. For example, an optimal DSS with slope $2/7$ is shown in Fig. \ref{fig:DSS1.1}; the basic set of points of this DSS are the points in the path from $(0, 0)$ to $(7, 2)$ and, because this basic set of points is repeated once, we can say that it has two periods. Furthermore, because the number of points of the basic set of points is $s+1$, the length of the period is $s$. Hence the length of the period of DSL in Fig. \ref{fig:DSS1.1} is 7.

It's easy to obtain all the different basic sets of points that define any rational DSS. We know that the value of the slope between two points of the rational DSS with distance $D_8=s$ is equal to $m$. So the first $s$ paths in a rational DSS of two periods that are defined by two points with distance $D_8=s$ with each other are basic sets of points for the DSS. For instance, if we extract the first 7 paths of the optimal DSS with slope $2/7$ that are shown in Fig. \ref{fig:DSS1.2}, then we will have all the different basic sets of points that define the DSS and are shown in Fig. \ref{fig:DSS1.3}.

\subsection{The slope of Digital Straight Lines}
\label{sec:DSLslope}

Lets define $\boldsymbol{\mu}_{Lp}(i)$ as the slope between a point $p$ of a digital line $L$ and its $i$ successive point in $L$ to the right of $p$, i.e.\ the point to the right of $p$ that have a distance $D_8=i$ with respect to $p$. From now on, if $L$ is a $DSS$ and $p$ is not explicitly defined, i.e.\  $\boldsymbol{\mu}_{L}(i)$, then $p$ will be the first point of $L$ from left to right.

When $L$ is an irrational digital line with slope $m$ and $p$ and $q$ are two distinct points that belong to $L$, we are certain of two things. The first thing we know is that, by definition, $\boldsymbol{\mu}_{Lp}(\infty)$ is equal to $m$. And the second thing is that there are not $\boldsymbol{\mu}_{Lp}(i)$ and $\boldsymbol{\mu}_{Lq}(i)$ equal for all the possible values of $i$, because the aperiodicity of $L$.

When $L$ is a rational digital line with slope $m$ equal to an irreducible fraction $r/s$, we know two things derived from its periodicity. If $p$ is any point of $L$ and $i$ is multiple of $s$, then $\boldsymbol{\mu}_{Lp}(i)$ is equal to $m$; it follows that if $q$ is any another point of $L$, then $\boldsymbol{\mu}_{Lp}(i)=\boldsymbol{\mu}_{Lq}(i)$. Finally, if $\{p_1,\ldots, p_s\}$ are any $s$ consecutive points belonging to $L$, then $\{\boldsymbol{\mu}_{Lp_1}, \dots, \boldsymbol{\mu}_{Lp_s}\}$ is the basic set that contains all the possible functions $\boldsymbol{\mu}_{Lq}$ for any point $q$.

For example, let $S_1$ and $S_3$ be two rational DSS that have two periods and are defined by the first and third basic set of points in Fig. \ref{fig:DSS1.3} respectively. So if we obtain $\boldsymbol{\mu}_{S_1}$ and $\boldsymbol{\mu}_{S_3}$ from their first point and plot their graphs, we'll have the slopes shown in Fig. \ref{fig:twoSlopes}. We can see in Fig. \ref{fig:twoSlopes} that $\boldsymbol{\mu}_{S_1}$ and $\boldsymbol{\mu}_{S_3}$ are equal to the represented slope $2/7$ at distances $D_8$ multiple of $7$, as we said in the last paragraph.

\begin{figure}[t]
\centering
\includegraphics[scale=0.65]{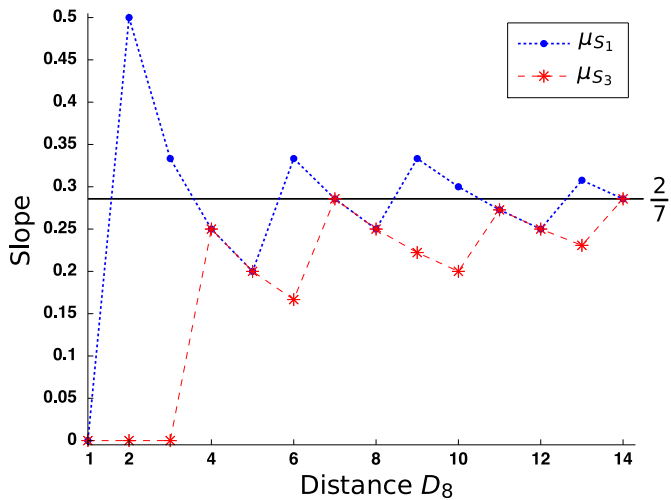}
\caption[$\boldsymbol{\mu}$ functions of two DSS that are defined by the first and third basic set of points in Fig. \ref{fig:DSS1.3} for their first two periods]{$\boldsymbol{\mu}$ functions of two DSS that are defined by the first and third basic set of points in Fig. \ref{fig:DSS1.3} for their first two periods.}
\label{fig:twoSlopes}
\end{figure}

\subsubsection{Boundaries of the slope}
\label{sec:slopeBoundaries}

In Fig. \ref{fig:twoSlopes} we can also observed that $\boldsymbol{\mu}_{S_1}$ and $\boldsymbol{\mu}_{S_3}$ tend to the value of the represented slope $2/7$, but we don't know the rate at which the discrete slope is reaching the real slope value. Therefore, the next two theorems will state how close from the real slope is the value of the $\boldsymbol{\mu}$ function given any distance $D_8$.

\begin{lemma}
Let $P(x_0,y_0)\in \mathbb{Z}^2$, $n\in \mathbb{N}$, $Q_0(x_0+n,y_0), Q_1(x_0+n,y_0+1), \ldots, Q_n(x_0+n,y_0+n)\in \mathbb{Z}^2$ be the points that have a distance $D_8=n$ in the first octant with respect to $P$, and $\overline{PQ_i}$ denote the slope between $P$ and $Q_i$. Then it is true that

\begin{equation}
\overline{PQ_{i+1}}-\overline{PQ_i}=\frac{1}{n}
\label{eq:lemma1}
\end{equation}
\end{lemma}

\begin{proof}

\noindent  The slopes between $P$ and $Q_0,\ldots,Q_n$ are

\small
\begin{equation}
\overline{PQ_{0}}=0, \overline{PQ_{1}}=\frac{1}{n}, \overline{PQ_2}=\frac{2}{n}, \ldots, \overline{PQ_n}=1
\label{eq:lemma1:1}
\end{equation}
\normalsize

\noindent Thus, we have
\begin{equation}
\overline{PQ_{i}}=\frac{i}{n}
\label{eq:lemma1:2}
\end{equation}

\noindent So, substituting equation \ref{eq:lemma1:2} in equation \ref{eq:lemma1} results

\begin{equation}
\overline{PQ_{i+1}}-\overline{PQ_i}=\frac{i+1}{n}-\frac{i}{n}=\frac{1}{n}
\label{eq:lemma1:3}
\end{equation}

\end{proof}

\begin{theorem}
\label{theorem:optimalBoundary}

Let the point  $P(x_0,y_0)$ belong to a digital line $L$ with slope $m$ such that  $0\leq m \leq1$, let the successive points of $P$ in $L$ be the optimal points with respect to $P$, and let the point $Q$ be any other point belonging to $L$ with a distance $D_8=n$ and a slope $\overline{PQ}$ with respect to $P$. Then, it is true that

\begin{equation}
\Big{|}m-\overline{PQ}\Big{|}\leq\frac{1}{2\cdot n}
\label{eq:theorem1}
\end{equation}

\end{theorem}

\begin{proof}

At a distance $D_8=n$ from $P$ the real line it's between $Q_{i}$ and $Q_{i+1}$. So, using an optimal criterion, $Q$ is equal to $Q_i$ if

\begin{equation}
\overline{PQ_{i}} \leq m < \frac{\overline{PQ_{i+1}}-\overline{PQ_i}}{2}+\overline{PQ_i}
\label{eq:theorem1:1}
\end{equation}

\noindent or $Q$ is equal to $Q_{i+1}$ if

\begin{equation}
\frac{\overline{PQ_{i+1}}-\overline{PQ_i}}{2} + \overline{PQ_i} \leq m < \overline{PQ_{i+1}}
\label{eq:theorem1:2}
\end{equation}

\noindent Now, when $Q$ equals $Q_{i}$, we substitute equation \ref{eq:lemma1} in inequality \ref{eq:theorem1:1}

\begin{equation}
\overline{PQ_{i}} \leq m < \frac{1}{2\cdot n}+\overline{PQ_{i}}
\label{eq:theorem1:3}
\end{equation}

\noindent substracting $\overline{PQ_{i}}$ from the inequality \ref{eq:theorem1:3}

\begin{equation}
0\leq m - \overline{PQ_{i}} < \frac{1}{2\cdot n}
\label{eq:theorem1:8}
\end{equation}

\noindent substituting $\overline{PQ}$ for $\overline{PQ_{i}}$ results

\begin{equation}
m - \overline{PQ} < \frac{1}{2\cdot n}
\label{eq:theorem1:9}
\end{equation}

\noindent wich proves the statement. For the last case, when $Q$ equals $Q_{i+1}$, we substract $\overline{PQ_{i+1}}$ from inequality \ref{eq:theorem1:2}

\begin{equation}
-\frac{\overline{PQ_{i+1}}-\overline{PQ_i}}{2} \leq m - \overline{PQ_{i+1}} < 0
\label{eq:theorem1:10}
\end{equation}

\noindent substituting equation \ref{eq:lemma1} in inequality \ref{eq:theorem1:10} we have

\begin{equation}
-\frac{1}{2\cdot n} \leq m - \overline{PQ_{i+1}} < 0
\label{eq:theorem1:11}
\end{equation}

\noindent Finally, substituting $\overline{PQ}$ for $\overline{PQ_{i+1}}$ in inequality \ref{eq:theorem1:11} and obtaining the absolute value results

\begin{equation}
|m-\overline{PQ}|\leq \frac{1}{2\cdot n}
\label{eq:theorem1:12}
\end{equation}

\end{proof}

The above theorem establish upper and lower boundaries for the $\boldsymbol{\mu}$ function of an optimal DSS. For instance, as before, let $S_1, \ldots, S_7$ be rational DSS with two periods and let them be defined by the seven basic set of points in Fig. \ref{fig:DSS1.3} respectively. Because $S_1$ is defined by the optimal set of points, the boundaries described by theorem \ref{theorem:optimalBoundary} are only valid for $\boldsymbol{\mu}_{S_1}$ and are shown in Fig. \ref{fig:slopeBoundaries:1}. The next corollary uses the boundaries of the slope described by theorem \ref{theorem:optimalBoundary} to extract two lines that serve as boundaries for the points of any DSS.

\begin{corollary}
\label{corollary:limitLines}
Let the point $P(x_0,y_0)$ in theorem \ref{theorem:optimalBoundary} belong to the real line defined by $y=m\cdot x + c$, and let $0 \leq a,b \leq 1/2 \in \mathbb{R}$. Then the points belonging to the digital line $L$ lay between two boundary lines defined

\begin{equation}
y=m\cdot x + c-b
\label{eq:corollary:limitLines:10}
\end{equation}

\noindent and

\begin{equation}
y=m\cdot x + c+a
\label{eq:corollary:limitLines:20}
\end{equation}

\end{corollary}

\begin{proof}
The absolute value of $m$ minus the slope between $P$ and the points in $y=m\cdot x + c-b$ is

\begin{equation}
b\cdot \Big{|}\frac{1}{x-x_0}\Big{|}
\label{eq:corollary:limitLines:30}
\end{equation}

\noindent and the absolute value of $m$ minus the slope between $P$ and the points in $y=m\cdot x + c+a$ is

\begin{equation}
a\cdot \Big{|}\frac{1}{x-x_0}\Big{|}
\label{eq:corollary:limitLines:40}
\end{equation}

The distance $D_8$ between $P$ and the points of both boundary lines is $x-x_0$, thus the term $x-x_0$ in equations \ref{eq:corollary:limitLines:30} and \ref{eq:corollary:limitLines:40} is the distance $D_8$.  By definition $0 \leq a,b \leq 1/2$, hence for all integer values of $x$ is true that

\begin{equation}
a\cdot \Big{|}\frac{1}{x-x_0}\Big{|},\; b\cdot \Big{|}\frac{1}{x-x_0}\Big{|} \leq \frac{1}{2}\cdot \Big{|}\frac{1}{x-x_0}\Big{|}
\label{eq::limitLines:3}
\end{equation}

\end{proof}

Continuing with our example, from corollary \ref{corollary:limitLines} it follows that the general boundary lines for the points belonging to $S_1$ are the lines defined by $y=2/7\cdot x-1/2$ and $y=2/7\cdot x+1/2$, but the exact boundary lines are defined by $y=2/7\cdot x-3/7$ and $y=2/7\cdot x+3/7$, as illustrated in Fig. \ref{fig:slopeBoundaries:3}. The next corollary uses the boundary lines defined by corollary  \ref{corollary:limitLines} to extract the boundaries of the slope for any DSS.

\begin{figure*}[t]
\centering
\subfloat[Slope boundaries for $\boldsymbol{\mu}_{S_1}$.]{\label{fig:slopeBoundaries:1} \includegraphics[scale=0.65]{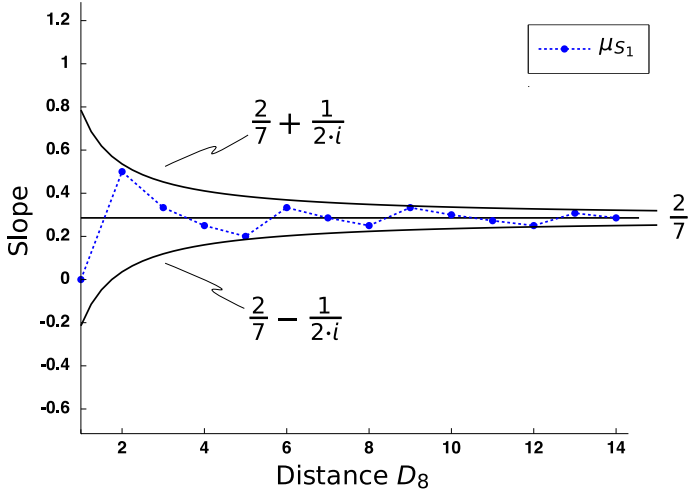} }
\subfloat[Exact boundary lines for $S_1$.]{\label{fig:slopeBoundaries:3} \hspace{0.37cm}\includegraphics[scale=0.65]{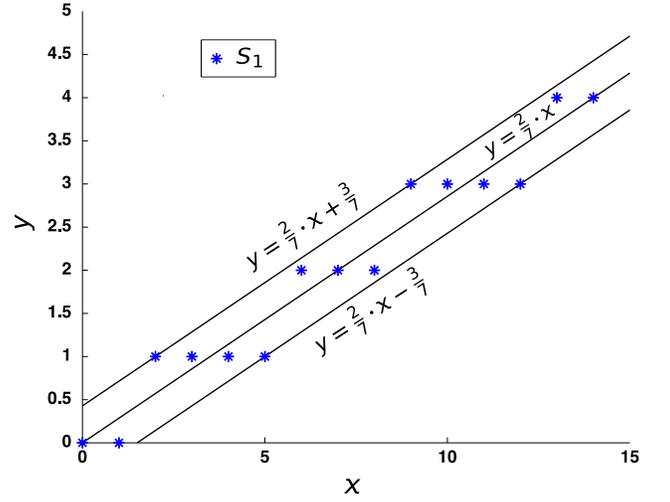} \hspace{0.37cm}}

\subfloat[Exact slope boundaries for $\boldsymbol{\mu}_{S_3}$.]{\label{fig:slopeBoundaries:4} \includegraphics[scale=0.65]{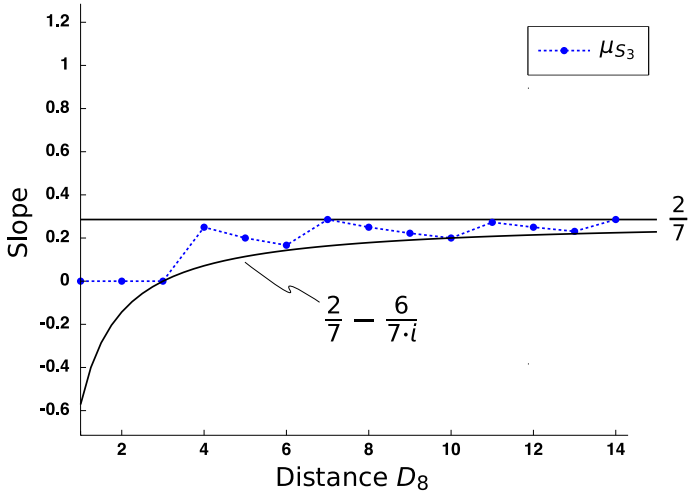} }
\subfloat[General slope boundaries for the $\boldsymbol{\mu}$ function of $S_1, \ldots, S_7$.]{\label{fig:slopeBoundaries:2} \includegraphics[scale=0.65]{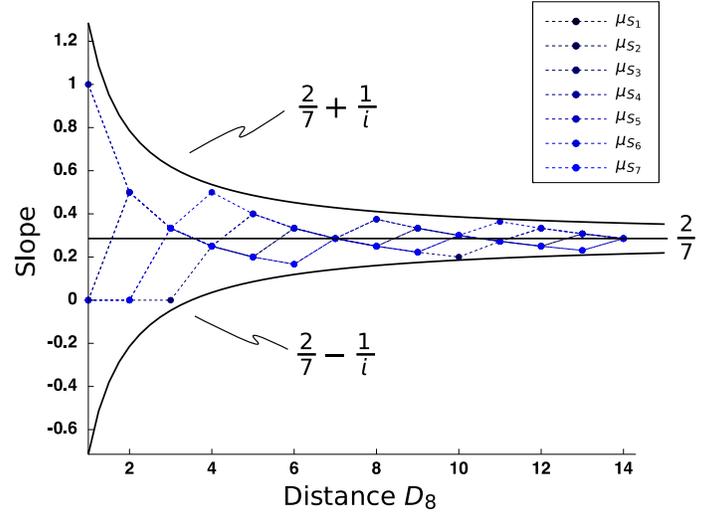} }
\caption[]{Slope and line boundaries for the family of basic sets that appear in Fig. \ref{fig:DSS1.3}.}
\label{fig:slopeBoundaries}
\end{figure*}

\begin{corollary}
\label{corollary:generalBoundaries}

Let $k \in \mathbb{R}$ be $-b \leq k \leq a$, let the point  $R(x_1, y_1)$ belong to $L$ and to the real line $y=m\cdot x + c+k$. Then the slope between $R$ and any other point belonging to $L$ lies between

\begin{equation}
m-(b+k)\cdot \frac{1}{x-x_1}
\label{eq:corollary:generalBoundaries:10}
\end{equation}

\noindent and

\begin{equation}
m+(a-k)\cdot \frac{1}{x-x_1}
\label{eq:corollary:generalBoundaries:20}
\end{equation}

\end{corollary}

\begin{proof}
The slope between $R$ and the lower limit line defined by equation \ref{eq:corollary:limitLines:10} is

\small
\begin{equation}
\frac{m\cdot x + c-b - (m\cdot x_1+c+k)}{x-x_1}=m-(b+k)\cdot \frac{1}{x-x_1}
\label{eq:corollary:generalBoundaries:30}
\end{equation}
\normalsize

\noindent and the slope between $R$ and the upper limit line defined by equation \ref{eq:corollary:limitLines:20} is

\small
\begin{equation}
\frac{m\cdot x + c+a - (m\cdot x_1+c+k)}{x-x_1}=m+(a-k)\cdot \frac{1}{x-x_1}
\label{eq:corollary:generalBoundaries:40}
\end{equation}
\normalsize

\end{proof}

This last corollary means that if we know the boundary lines of a particular DSS, we can extract its specific slope boundaries. For example, from theorem \ref{theorem:optimalBoundary} we already have a slope boundary for $\boldsymbol{\mu}_{S_1}$, but now that we know the line boundaries of $S_1$ we can extract its exact slope boundaries that are $\frac{2}{7}-\frac{3}{7\cdot i}$ and $\frac{2}{7}+\frac{3}{7\cdot i}$. In the case of a non-optimal DSS like $S_3$ we can also extract its specific slope boundaries, which are $\frac{2}{7}-\frac{6}{7\cdot i}$ and $\frac{2}{7}$ as illustrated in Fig. \ref{fig:slopeBoundaries:4}. The next corollary it's simple and important for formulation of the method describe in this paper.

\begin{corollary}
\label{corollary:sizeOfGeneralBoundaries}
The absolute value of the difference between the upper and lower slope boundaries defined in corollary \ref{corollary:generalBoundaries} is less or equal than $\Big{|}\frac{1}{x-x_1}\Big{|}$
\end{corollary}

\begin{proof}
The absolute value of the difference between the upper and lower slope boundaries is

\begin{equation}
(a+b)\cdot\Big{|}\frac{1}{x-x_1}\Big{|}
\label{eq:corollary:sizeOfGeneralBoundaries:10}
\end{equation}

\noindent By definition $a+b \leq 1$, then

\begin{equation}
(a+b)\cdot\Big{|}\frac{1}{x-x_1}\Big{|} \leq \Big{|}\frac{1}{x-x_1}\Big{|}
\label{eq:corollary:sizeOfGeneralBoundaries:20}
\end{equation}

\end{proof}

This last corollary states how far from the real slope could be the value of a $\boldsymbol{\mu}$ function of any DSS, i.e.\ corollary \ref{corollary:sizeOfGeneralBoundaries} establish the general upper and lower boundaries for the $\boldsymbol{\mu}$ function of any DSS. Let the point  $P$ and $Q$ belong to a digital line $L$ with slope $m$ such that  $0\leq m\leq1$,  and let the point $Q$ be any other point belonging to $L$ with a distance $D_8=n$ and a slope $\overline{PQ}$ with respect to $P$. Then, by corollary \ref{corollary:sizeOfGeneralBoundaries} it's true that

\begin{equation}
\Big{|}m-\overline{PQ}\Big{|}\leq\frac{1}{n}
\label{eq:generalBoundary}
\end{equation}

As an example, in Fig. \ref{fig:slopeBoundaries:2} we can see the general slope boundaries for all the $\boldsymbol{\mu}$ function of $S_1, \ldots, S_7$.

\subsection{DSS Determination Methods}

\subsubsection{Determination of two connected DSS using the slope}
\label{sec:particularMethod}

It's easy to determine real line segments from a given sequence of points using the slope, because the slope between any two different points of the real line is always the same. Due this, if $p_0, p_1, \ldots, p_n \in \mathbb{R}^2$ are the sequence of points that describes two connected line segments and we need to find the vertex that connects them, then we just simply calculate the slope $m_1, \ldots, m_n$ of $p_1, \ldots, p_n$ with respect to $p_0$ and the vertex will be the point with slope $m_i$ different from $m_1,\ldots, m_{i-1}$. As an example of this, Fig. \ref{fig:realMethod} shows two connected line segments and their vertex obtained using the slope between the point $(0, 0)$ and the its successive points in the segments.

\begin{figure}[t]
\centering
\includegraphics[scale=0.5]{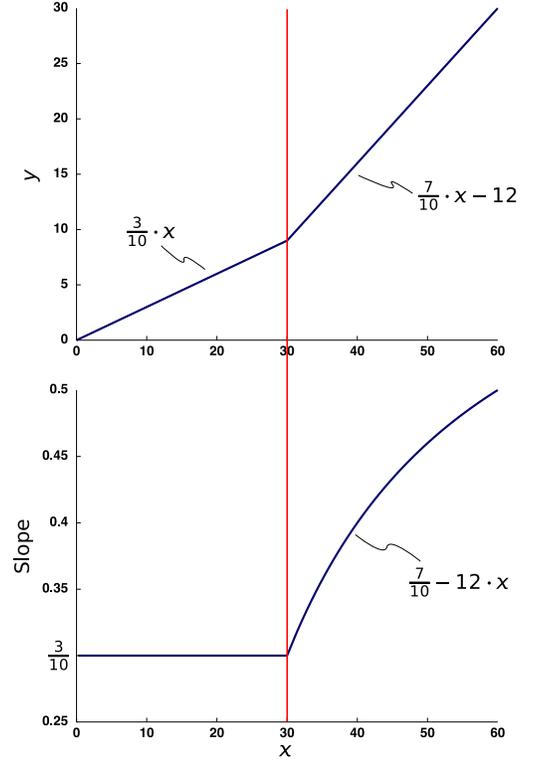} 
\caption[Two connected real line segments (upper), and the slope between the point $(0, 0)$ and its successive points in the segments (lower)]{Two connected real line segments (upper), and the slope between the point $(0, 0)$ and its successive points in the segments (lower).}
\label{fig:realMethod}
\end{figure}

This approach for finding real line segments from a sequence of points could be adapted for digital lines using the theorems stated in section \ref{sec:slopeBoundaries}. So let $\rho_0(x_0, y_0), \rho_1(x_1, y_1), \ldots, \rho_n(x_n, y_n) \in\mathbb{Z}^2$ be a sequence of points with distance $D_8=1$ between any two consecutive points, and let $\mu(i)$ be the slope between $\rho_0$ and $\rho_i$. We assume that $\rho_0, \rho_1, \ldots, \rho_n$ form one or more DSS.

\begin{figure*}[t]
\centering
\subfloat[Unknown slope boundaries of $\mu$.]{\label{fig:methodExample:1} \includegraphics[scale=0.7]{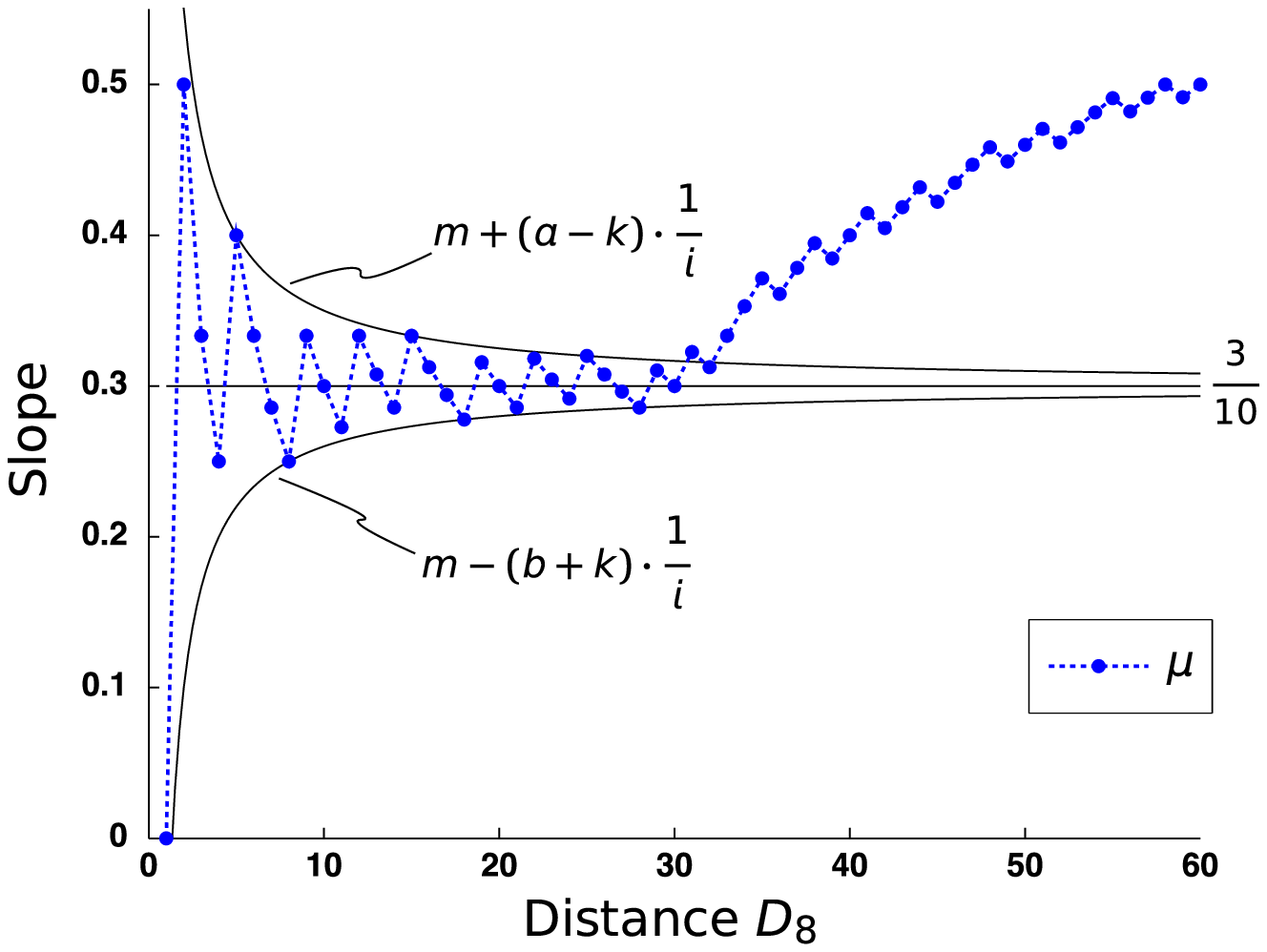} }
\subfloat[First slope boundaries based on $\mu$.]{\label{fig:methodExample:2} \includegraphics[scale=0.7]{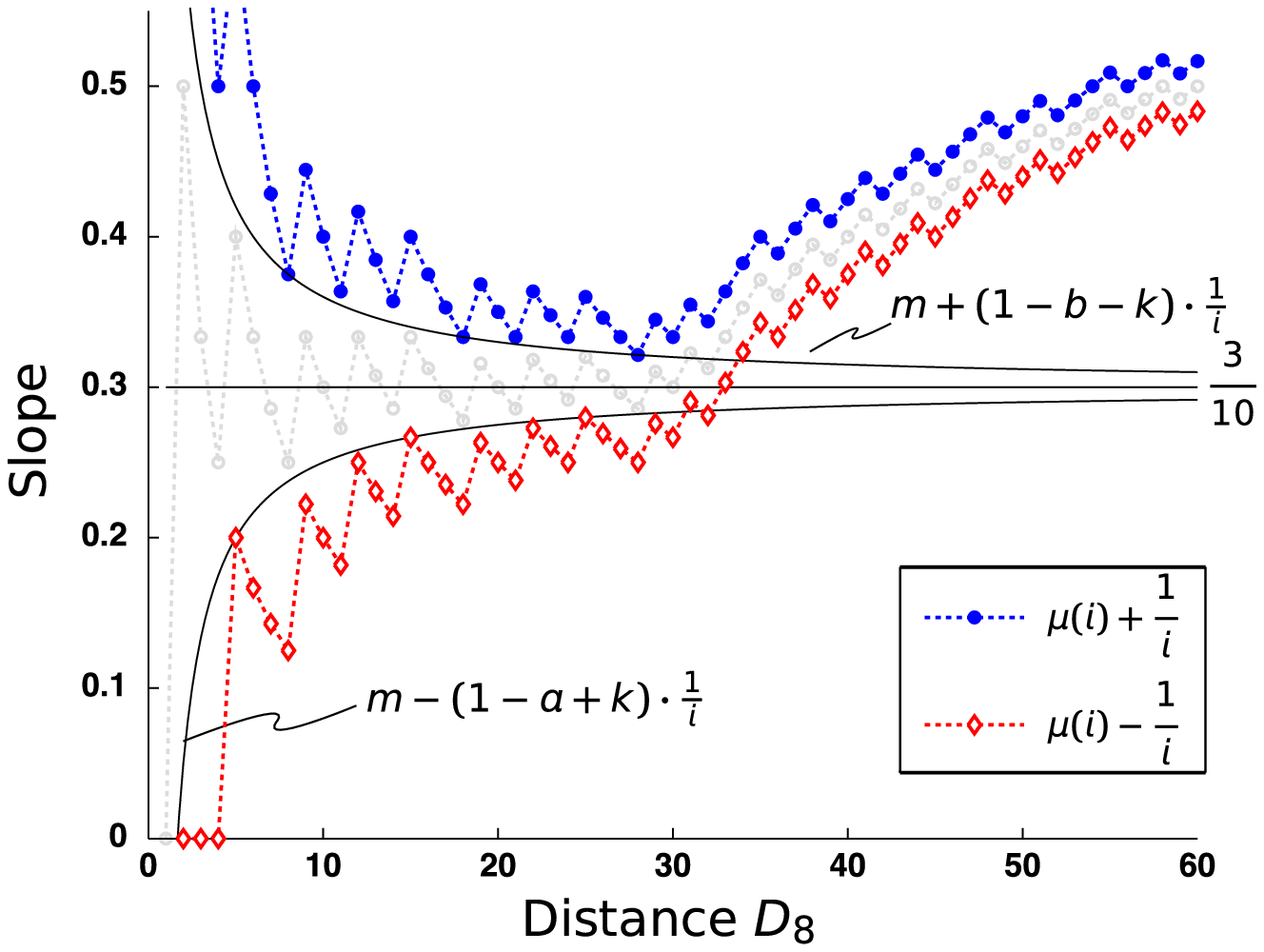} }

\subfloat[Final slope boundaries based on $\mu$.]{\label{fig:methodExample:3} \includegraphics[scale=0.7]{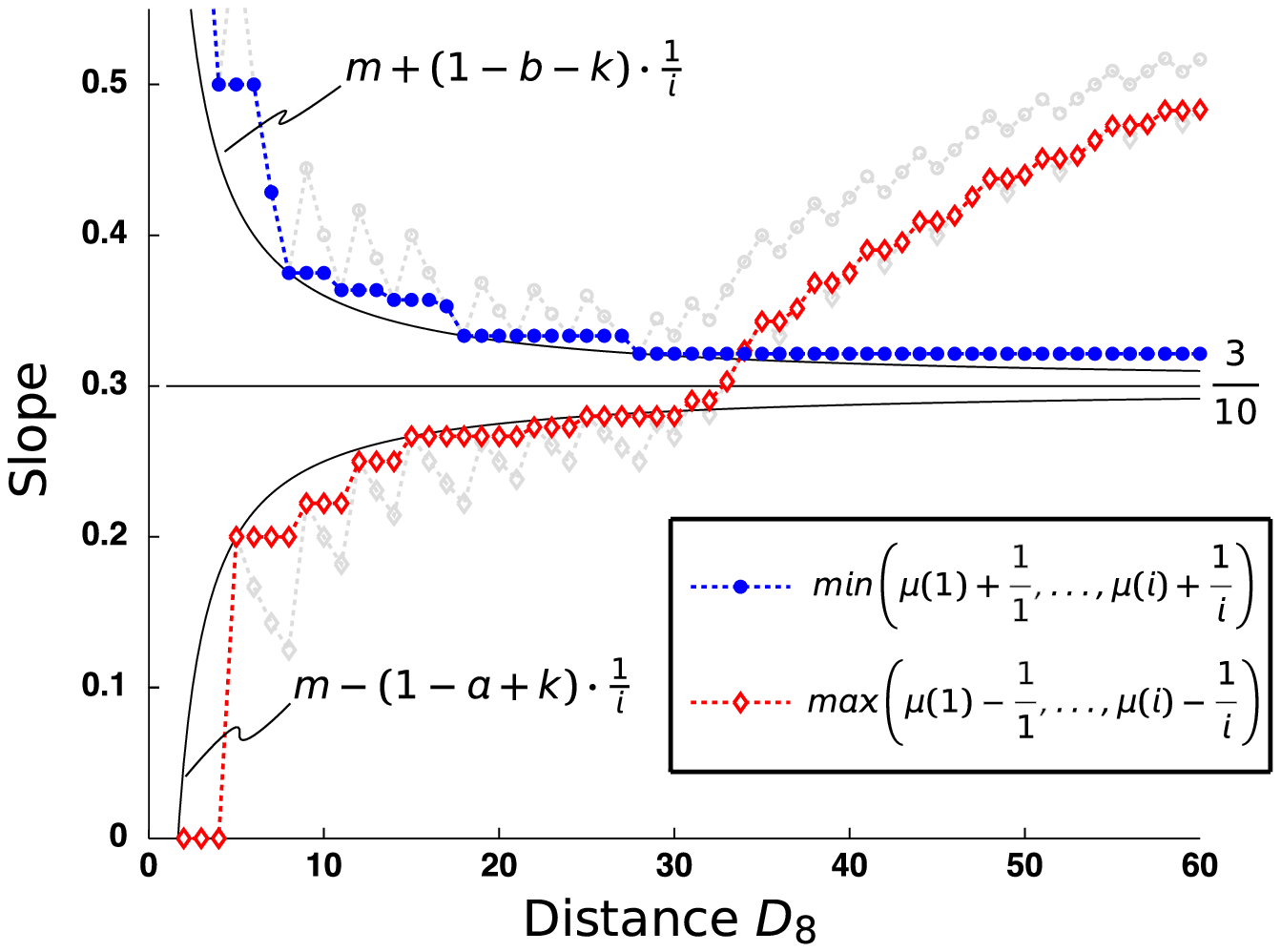} }
\subfloat[The difference between the final slope boundaries.]{\label{fig:methodExample:4} \hspace{0.3661cm}\includegraphics[scale=0.7]{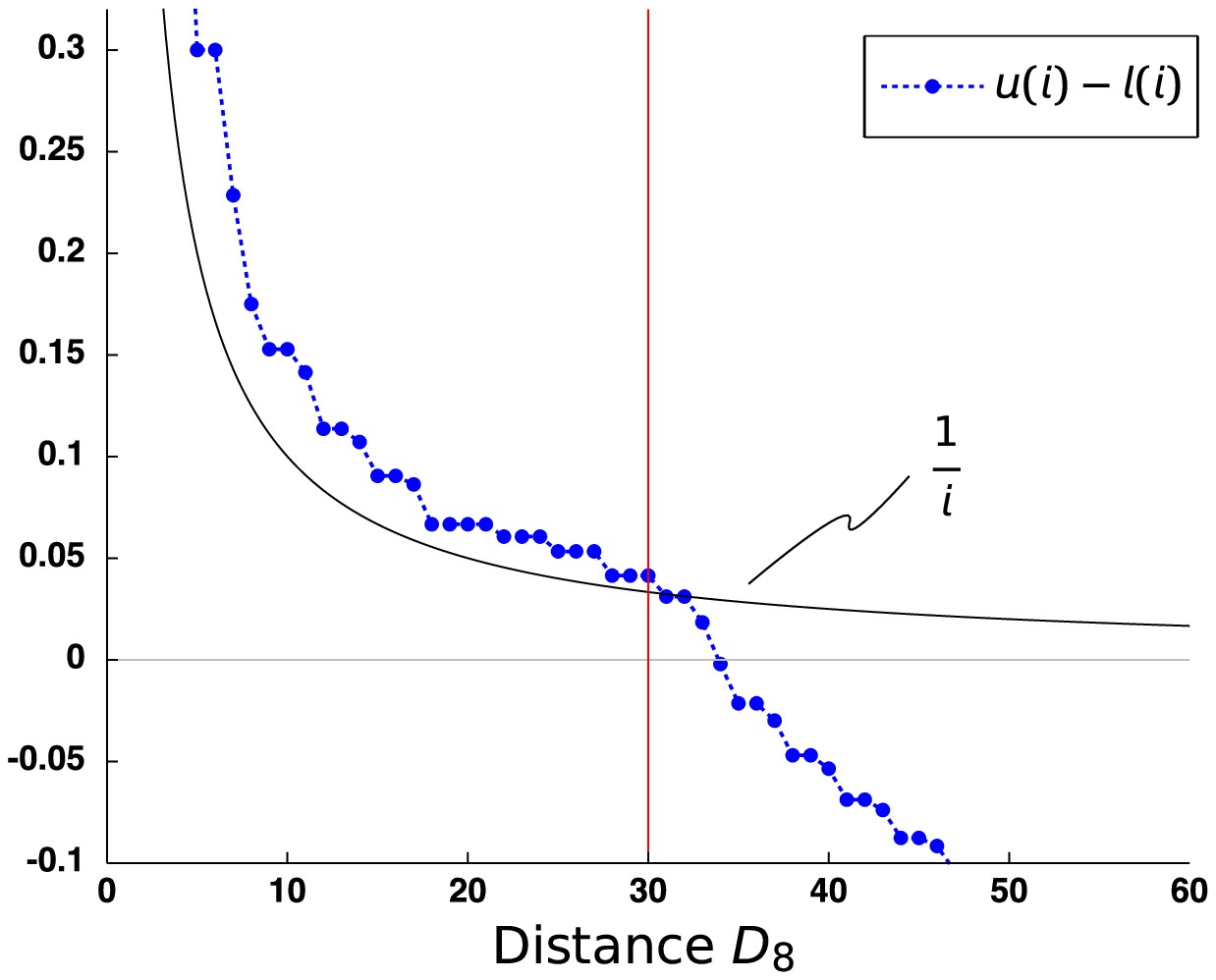} \hspace{0.3661cm}}
\caption[]{Separation of the digital counterpart of the two connected segments of Fig. \ref{fig:realMethod} based on its slope.}
\label{fig:methodExample}
\end{figure*}

We need to extract something that resembles slope boundaries from the points of $\mu$. Let $\rho_{\eta}$ be the last point of the first DSS in $\rho_0(x_0, y_0), \rho_1(x_1, y_1), \ldots, \rho_n(x_n, y_n)$. From corollary \ref{corollary:generalBoundaries} we know that for the first $\eta$ values of $\mu$ is true that

\begin{equation}
\mu(i) \geq m-(b+k)\cdot \frac{1}{i}
\label{eq:method:10}
\end{equation}

\noindent and

\begin{equation}
\mu(i) \leq m+(a-k)\cdot \frac{1}{i}
\label{eq:method:20}
\end{equation}

\noindent For example, Fig. \ref{fig:methodExample:1} shows that the first 30 values of the digital counterpart of the slope shown in Fig. \ref{fig:realMethod} are between those boundaries.

In order to obtain new slope boundaries from $\mu$, lets add $1/i$ to equation \ref{eq:method:10}, thus we have

\begin{equation}
\mu(i)+\frac{1}{i}\geq m+(1-b-k)\cdot \frac{1}{i}
\label{eq:method:30}
\end{equation}

\noindent and lets subtract $1/i$ to equation \ref{eq:method:20}

\begin{equation}
\mu(i)-\frac{1}{i}\leq m-(1-a+k)\cdot \frac{1}{i}
\label{eq:method:40}
\end{equation}

\noindent We can see in Fig.\ref{fig:methodExample:2} that both conditions still hold for the first 30 values of $\mu$. Now, let $l$ and $u$ represent the lower and upper slope boundary respectively, let $l(i)$ be defined as

\begin{equation}
l(i)=max\Bigg{(}\mu(1)-\frac{1}{1}, \ldots,\mu(i)-\displaystyle \frac{1}{i}\Bigg{)}
\label{eq:method:60}
\end{equation}

\noindent and let $u(i)$ be defined as

\begin{equation}
u(i)=min\Bigg{(}\mu(1)+\frac{1}{1}, \ldots,\mu(i)+\displaystyle \frac{1}{i}\Bigg{)}
\label{eq:method:70}
\end{equation}

\noindent These new boundaries appear in Fig. \ref{fig:methodExample:3}. From equations \ref{eq:method:30} and \ref{eq:method:40} we know is true that

\begin{equation}
\mu(i)+\frac{1}{i}\geq u(i)\geq m+(1-b-k)\cdot \frac{1}{i}
\label{eq:method:80}
\end{equation}

\noindent and

\begin{equation}
\mu(i)-\frac{1}{i}\leq l(i)\leq m-(1-a+k)\cdot \frac{1}{i}
\label{eq:method:90}
\end{equation}

Using equations \ref{eq:method:80} and \ref{eq:method:90} we can find the next relation about the difference between $u(i)$ and $l(i)$

\begin{equation}
u(i)-l(i)\geq \frac{2-(a+b)}{i}
\label{eq:method:100}
\end{equation}

Finally, from corollary \ref{corollary:sizeOfGeneralBoundaries} is true that

\begin{equation}
u(i)-l(i)\geq \frac{2-(a+b)}{i}\geq\frac{1}{i}
\label{eq:method:110}
\end{equation}

Therefore, only the consecutive points that satisfy this last equation belong to the first DSS, and thus, we can use it to find the vertex of two connected DSS. For instance, Fig. \ref{fig:methodExample:4} shows that the difference between the slope boundaries is less than $1/i$ starting at a distance $D_8$ equal to 31, so the vertex is the point $\rho_{30}$.

\subsubsection{Freeman's Chain Coding Scheme}

Freeman\cite{Freeman1970} describes the borders of objects quantizing the direction of pairs of consecutive pixels by elements coded with the integers $0,1, \ldots, 7$. For example, Fig. \ref{fig:freemanChainCodeScheme} shows a chain-coded boundary. Freeman states in \cite{Freeman1961} that chains of straight lines must possess the next specific properties: 
\begin{enumerate}
\item At most two types of symbols can be present, and these can differ only by unity, modulo eight.
\item One of the two symbols always occurs singly.
\item Successive occurrences of the single symbol are as uniformly spaced as possible.
\end{enumerate}

It's important to note that the third property is not precisely formulated (for a better and formal specification of the third property see \cite{Rosenfeld1974}).

\begin{figure}[h]
\centering
\includegraphics[scale=0.95]{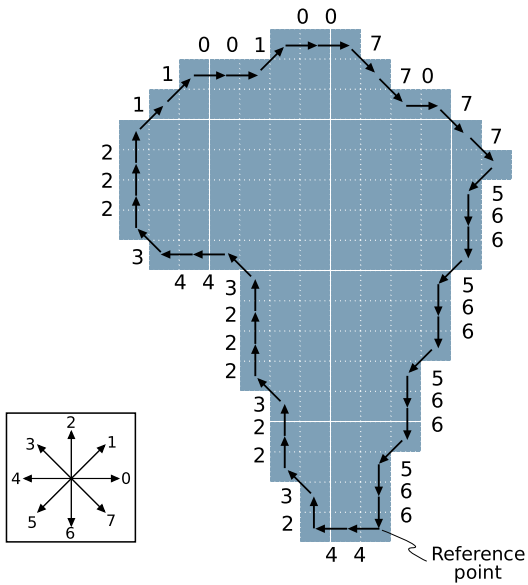}
\caption[Chain coding scheme (lower left), and a chain-coded boundary]{Chain coding scheme (lower left), and the boundary with a chain code equal to 44232232223443222110010077077566566566566.}
\label{fig:freemanChainCodeScheme}
\end{figure}
%

\section{Tracing Digital Straight Segments using the Slope}
\label{sec:method}

To solve the problem stated, our method combines the first two properties about chain codes and the general version of the method introduced in section \ref{sec:particularMethod}. The method returns a list of vertexes $v_1, v_2, \ldots, v_m$ between the found DSS. The method iterates throw every point in $\rho_0, \ldots, \rho_n \in\mathbb{Z}^2$, and for each point $i$ obtains its chain code and evaluates if the first two properties are true. If any of the two properties is false, it means that $\rho_i$ is one vertex, and we continue evaluating the next points. When both properties are true for $\rho_i$ we still don't know if is a vertex or not, so we evaluate $\rho_i$ using the slope.

The general version of the method presented in section \ref{sec:particularMethod} changes two things from the original. First, the function $\mu$ is the slope between the point $\rho_i$ and the last vertex found $v_j$ and the distance $D_8$ considered is between them. Second, if the slope between $\rho_i$ and $v_j$ is not in $[-1, 1]$, then the value of $\mu$ for $\rho_i, \rho_{i+1}, \ldots$ and $v_j$ is the inverse of their slope.

\begin{figure}[t]
\centering
\includegraphics[scale=0.72]{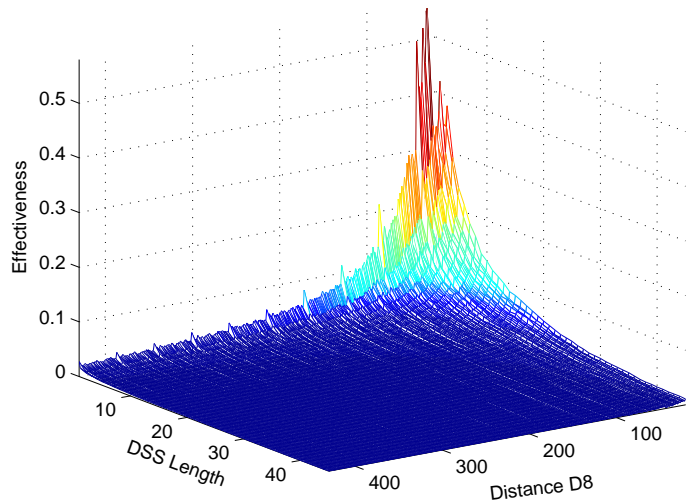} 
\caption{Average of effectiveness}
\label{fig:average}
\end{figure}

\section{Tests}
\label{sec:tests}

The tests we performed were intended to measure the effectiveness of our method in finding the exact vertex between two connected DSS of the same length $l$. As a metric of effectiveness we employed the division of the distance $D_8$ between the real vertex and the found vertex by the length $l$. This metric is in the range of 0 to 1, where 0 means a perfect match and 1 means total failure.

Our data bank consisted in 811800 cases of two connected rational DSS. To build our synthetic data bank we first chose the slopes of the first segment to evaluate, in this case were the slopes: $0, 1/43, 2/43, \ldots, 43/43$. It's important to mention that the denominator had to be a prime number so all the first DSS had the same length of period. Then for every slope of the first DSS we calculate 45 slopes for the second DSS that differ from them by $1^{\circ}, 2^{\circ}, \ldots, 45^{\circ}$. Finally, the length $l$ for both DSS in all cases were in the range of 20 and 430, so we could evaluate them from their half period to their tenth.

\section{Conclusions and further work}

We have presented a novel method for the determination of DSS using the slope. Our method its based upon some unseen properties of the slope. We believe that this properties are important because they could led to find new properties about the DSS and its description. They already could be used to proof in an easier fashion the chord property stated by Rosenfeld in \cite{Rosenfeld1974}. Besides this, it seems that the family of slope $\boldsymbol{\mu}$ of any DSS could be described using the exact slope boundaries introduced in this paper and a set of functions like $\alpha, \alpha+\beta/i,  \alpha+2\cdot\beta/i, \ldots$ For instance, Fig. \ref{fig:conclusion} shows that all the possible values of $\boldsymbol{\mu}$ that appear in Fig. \ref{fig:slopeBoundaries:2} could be obtain from the intersection of $0, 1/i, 2/i, 3/i, 4/i$ and the area between $2/7-3/7\cdot i$ and $2/7+3/7\cdot i$.

Our method needs to be tested against other methods in order to compare their effectiveness and performance. We also want to extend our work for the evaluation of the digital curvature.

\begin{figure}[h]
\centering
\includegraphics[scale=0.7]{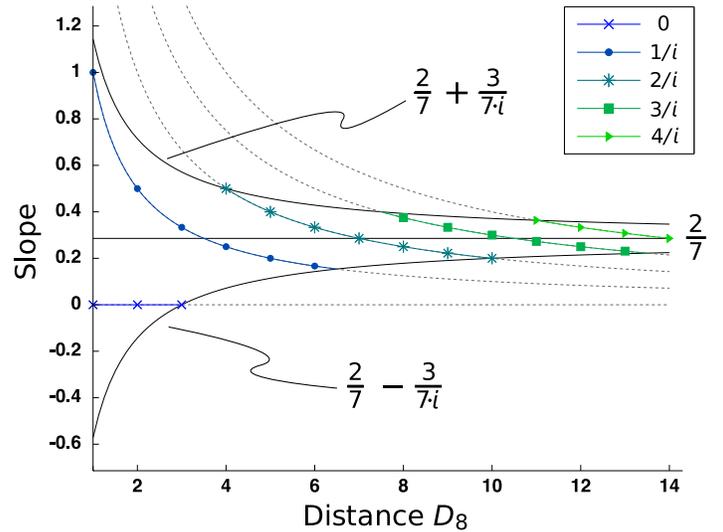}
\caption{All the possible values of the $\boldsymbol{\mu}$ function of a DSS with slope $2/7$ for their two first periods.}
\label{fig:conclusion}
\end{figure}

\end{document}